%% file: main.tex
\documentclass[letterpaper]{article}
\usepackage{aaai20}
\usepackage{fixbib}
\usepackage{times}
\usepackage{helvet}
\usepackage{courier}

\usepackage{amsmath}
\usepackage{amsthm}
\usepackage[utf8]{inputenc} 
\usepackage[T1]{fontenc}    
\usepackage{url}            
\usepackage{booktabs}       
\usepackage{amsfonts}       
\usepackage{nicefrac}       
\usepackage{microtype}      
\usepackage{mathtools}
\usepackage{mathrsfs}
\usepackage{multirow}
\usepackage{enumitem}
\usepackage{graphicx,subcaption}
\usepackage{newpxmath}
\usepackage{color}

\theoremstyle{definition}
\newtheorem{definition}{Definition}
\newtheorem{prop}{Proposition}
\newtheorem{thm}{Theorem}

\begin{document}
%
\title{Connecting First and Second Order Recurrent Networks with \\Deterministic Finite Automata}
\author{Qinglong Wang \\ McGill University
\And Kaixuan Zhang \\ Pennsylvania State University
\And Xue Liu \\ McGill University
\And C. Lee Giles \\ Pennsylvania State University
}
\maketitle
\input{base/0_abstract.tex}
\input{base/1_intro.tex}
\input{base/2_rw.tex}
\input{base/3_complexity.tex}
\input{base/4_rnn.tex}
\input{base/5_eval.tex}

\input{base/6_conclusion.tex}

\begin{quote}
\begin{small}
\bibliographystyle{aaai}
\bibliography{ref}
\end{small}
\end{quote}

\end{document}

%% file: base/0_abstract.tex
\begin{abstract}
We propose an approach that connects recurrent networks with different orders of hidden interaction with regular grammars of different levels of complexity. 
We argue that the correspondence between recurrent networks and formal computational models gives understanding to the analysis of the complicated behaviors of recurrent networks. 
We introduce an entropy value that categorizes all regular grammars into three classes with different levels of complexity, and show that several existing recurrent networks match grammars from either all or partial classes. 
As such, the differences between regular grammars reveal the different properties of these models. 
We also provide a unification of all investigated recurrent networks. 
Our evaluation shows that the unified recurrent network has improved performance in learning grammars, and demonstrate comparable performance on a real-world datasets with more complicated models.
\end{abstract}

%% file: base/1_intro.tex
\section{Introduction}
\label{sec:intro}
Over time, many different recurrent neural networks (RNNs) have been proposed, including the simple Elman network~\cite{elman1990finding}(also referred as simple recurrent network (SRN)), many enhanced models (second-order RNN (2-RNN)~\cite{giles1992learning}, multiplicative RNN (M-RNN)~\cite{sutskever2011generating}, multiplicative integration RNN (MI-RNN)~\cite{wu2016multiplicative}, RNNs with long-short-term-memory (LSTM)~\cite{hochreiter1997long} and gated-recurrent-unit (GRU)~\cite{cho2014properties}, etc) and have been used in many machine learning tasks that involve sequential data, e.g. financial forcasting ~\cite{giles2001noisy}l language processing, speech recognition, and program analysis~\cite{irie2016lstm,fu2016using}. 
However, these models are difficult to inspect, analyze, and verify due to their black box nature.

Recent work attempts to address this by establishing both theoretical and empirical connections between RNNs and finite state machines and grammars. Surprisingly, Minsky early on had proposed such connections~\cite{minsky1967computation}. 
Theoretically, it has been shown that certain RNN -- 2-RNN with linear hidden activation -- is equivalent to weighted automata~\cite{rabusseau2018connecting}. 
Empirically, much prior work~\cite{ZengGS93,weiss2017extracting,wang2018verification,MichalenkoSVBCP19,Merrill19} has presented different ways to extract deterministic finite automata (DFA) from trained RNNs. This line of research has led to using extracted DFAs for interpreting~\cite{weiss2017extracting} and verifying~\cite{wang2018verification} RNNs. 
Motivation is that formal computational models, especially DFAs, are well studied and have been previously used for the same purpose~\cite{jacobsson2005rule}. 
More importantly, the fact that certain types of RNNs can more readily learn certain types of formal languages may provide crucial insight in understanding and analyzing RNNs. 
This work establishes closer connections between different RNNs and different classes of regular languages from both theoretical and empirical perspectives. 
Specifically, we propose novel approaches to measure the complexity of regular languages and categorize them accordingly. 
We then investigate different RNNs for their properties for learning different classes of regular languages. 
Lastly, we empirically validate our analysis on different types of regular languages and a real-world dataset.

%% file: base/2_rw.tex
\section{Preliminaries}
\label{sec:rw}
\subsection{Recurrent Neural Networks} 
We present a unified view of the update activity of recurrent neurons for different RNNs we investigate (shown in Table~\ref{tab:rnn_models}). 
Typically, a RNN consists of a hidden layer $h$ containing $N_h$ recurrent neurons (each designated as $h_i$), and an input layer $x$ containing $N_x$ input neurons (each designated as $x_k$). 
We denote the values of $h$ at $t\,$th and $t\!-\!1\,$th discrete times as $h^{t}$ and $h^{t-1}$. Then the hidden layer is updated by:\\
\centerline{$h^{t+1} = \phi(x^t, h^{t}, W)$,}\\
where $\phi$ is the activation function (e.g. {\tt Tanh} and {\tt Relu}.), and $W$ denotes the weights which modify the strength of interaction among input neurons, hidden neurons, output neurons, and any other auxiliary units. 
The hidden layer update for each RNN is presented in Table~\ref{tab:rnn_models}.

\paragraph{SRN} (Elman network)~\cite{elman1990finding} integrates the input layer and the previous hidden layer in a manner that is regarded as a ``first-order'' connection~\cite{goudreau1994first}. 
This first-order connection has been widely adopted for building different recurrent layers, for example, the gate units in LSTM~\cite{hochreiter2001gradient} and GRU~\cite{cho2014properties}.

\paragraph{Higher-order RNNs} (such as tensor RNNs) have higher-order connections in their recurrent layers and are designed to capture more complex interactions between neurons. 
The 2-RNN~\cite{giles1992learning} has a recurrent layer updated by a weighed product of input and hidden neurons. 
This type of connection enables a direct mapping between 2-RNN and a DFA~\cite{omlin1996stable}. 
Recent work~\cite{rabusseau2018connecting} also shows the equivalence between a 2-RNN with linear hidden activation and weighted automata. 
Since a 2-RNN has a 3-D tensor weight, computation is more intensive. As such various approximations (M-RNN with a tensor decomposition~\cite{SutskeverMH11} and MI-RNN with a rank-1 approximation~\cite{wu2016multiplicative}) have been proposed to alleviate the computational cost while preserving the benefits of high order connections. 

\paragraph{RNNs with gated units} (e.g., LSTM~\cite{hochreiter2001gradient}, GRU~\cite{cho2014properties}) were proposed to deal with the vanishing and exploding gradient problems suffered by SRNs. 
While these RNNs are effective for capturing the long-term dependence between sequential inputs, their gate units induce highly nonlinear behavior to the update of the hidden layer which creates difficulty in analysis.

\subsection{Complexity of a Regular Grammar}
A regular grammar (RG) recognizes and generates a regular language -- a set of strings of symbols from an alphabet, and is uniquely associated with a DFA with a minimal number of states. 
A DFA covers a wide range of languages which means that all languages whose string length and alphabet size can be bounded can be recognized and generated by a DFA~\cite{giles1992learning}. 
Here we briefly revisit several prior approaches that measure the complexity of a RG. 

\paragraph{Complexity of Shift Space} In symbolic dynamics~\cite{lind1995introduction}, a particular form of entropy is defined to measure the ``information capacity'' of the \emph{shift space}, which is a set of bi-infinite symbolic sequences that represent the evolution of a discrete system. 
When applied to measure the complexity of a RG, this entropy describes the cardinality of the strings defined by its language.

\paragraph{Logical Complexity} RG can also be categorized according to logical complexity~\cite{rogers2013cognitive}: 
Strictly Local (SL), Strictly Piecewise (SP) (examples shown in Table~\ref{tab:data}), Locally Testable (LT), etc. 
These classes have multiple characterizations in terms of logic, automata, regular expressions, and abstract algebra~\cite{avcu2017subregular}. 
SL and SP languages are the simplest and most commonly used languages that define a finite set of factors (consecutive symbols) and subsequences, respectively and 
are selected to evaluate different RNNs on their performance in capturing the long-term dependency~\cite{avcu2017subregular}. 

%% file: base/3_complexity.tex
\section{Categorization of Regular Grammars}
\label{sec:complexity}
Here we introduce a particular entropy for measuring the complexity of a RG. We will use the commonly used Tomita grammars~\cite{tomita1982learning} as examples for presenting the analysis and the advantages of our entropy over the entropy defined in symbolic dynamics. We then categorize all RGs into three classes according to their entropy values. Last, we provide an efficient approach to compute the entropy of a RG by analyzing the transition matrix of its associated DFA.

\begin{table*}[t]
\small
\centering
\caption{Descriptions of selected regular grammars. $\ltimes$ ($\rtimes$) is the left (right) string boundary marker.}
\label{tab:data}
\begin{tabular}{lllll}
\hline
\hline
Language & $\Sigma$  & Description & Min. DFA Size &  \\ \hline \hline

SL-$4$ & 
\multirow{2}{*}{\begin{tabular}[l]{@{}l@{}}$\{a,b$,\\$c,d\}$\end{tabular}} & \textbf{Forbidden factors}: $\ltimes bbb$, $aaaa$, $bbbb$, $aaa\rtimes$  & 7  &  \\
SP-$8$ &  & \textbf{Forbidden subsequences}: $abbaabba$ & 8 &  \\ \cline{1-4}

\multirow{3}{*}{Tomita} & \multirow{3}{*}{$\{0,1\}$}     & \textbf{1}:$1^{*}$, \textbf{2}:$(1 0)^{*}$, \textbf{7}:$0^{*}1^{*}0^{*}1^{*}$ 
& \textbf{1}:2, \textbf{2}:3, \textbf{7}:5  &  \\
&  & \begin{tabular}[c]{@{}l@{}}\textbf{3}:an odd number of consecutive 1's is always\\ followed by an even number of consecutive 0's\\ \textbf{4}:any string not containing ``000'' as a substring\end{tabular} & \begin{tabular}[c]{@{}l@{}}\textbf{3}:5, \\ \textbf{4}:4\end{tabular} &  \\
& & \begin{tabular}[c]{@{}l@{}}\textbf{5}:even number of 0s and even number of 1's\\ \textbf{6}:the difference between the number of 0's and\\ the number of 1's is a multiple of 3\end{tabular} & \begin{tabular}[c]{@{}l@{}}\textbf{5}:4,\\ \textbf{6}:3\end{tabular}  &  \\ \cline{1-4}

STAMINA & -- & \begin{tabular}[c]{@{}l@{}}No. Problem --- Alphabet Size --- Sparsity\\ 81:85---50---100\%; 86:90---50---50\%\\91:95---50---25\%; 96:100---50---12.5\%\end{tabular}  & --   &  \\ \hline \hline 
\end{tabular}
\vspace{-1.0em}
\end{table*}

\begin{figure*}
\centering
\begin{subfigure}{.3\textwidth}
  \centering
  \includegraphics[width=.7\linewidth]{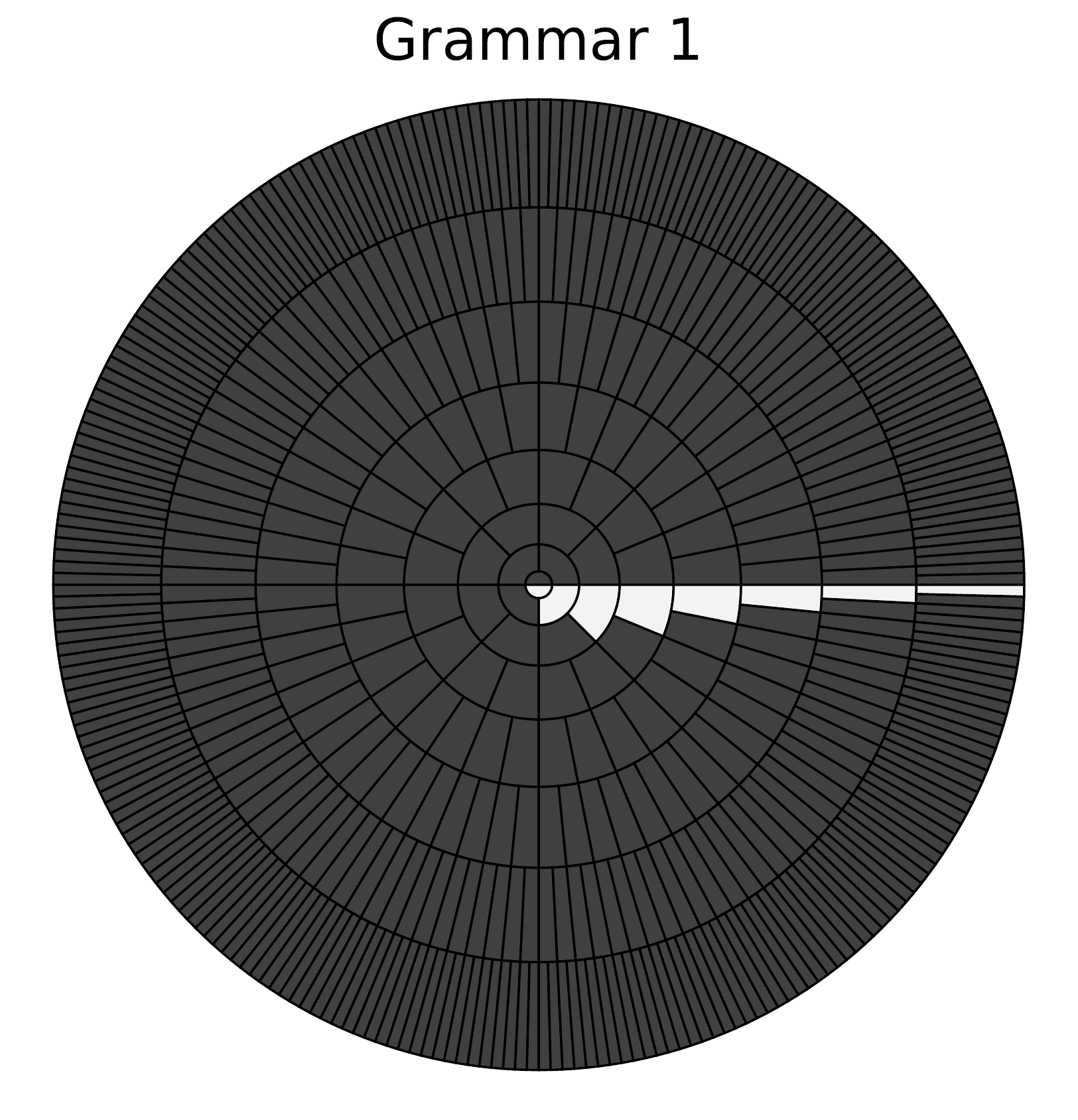}
  \label{fig:g2_pie}
\end{subfigure} \hfill 
\begin{subfigure}{.3\textwidth}
  \centering
  \includegraphics[width=.7\linewidth]{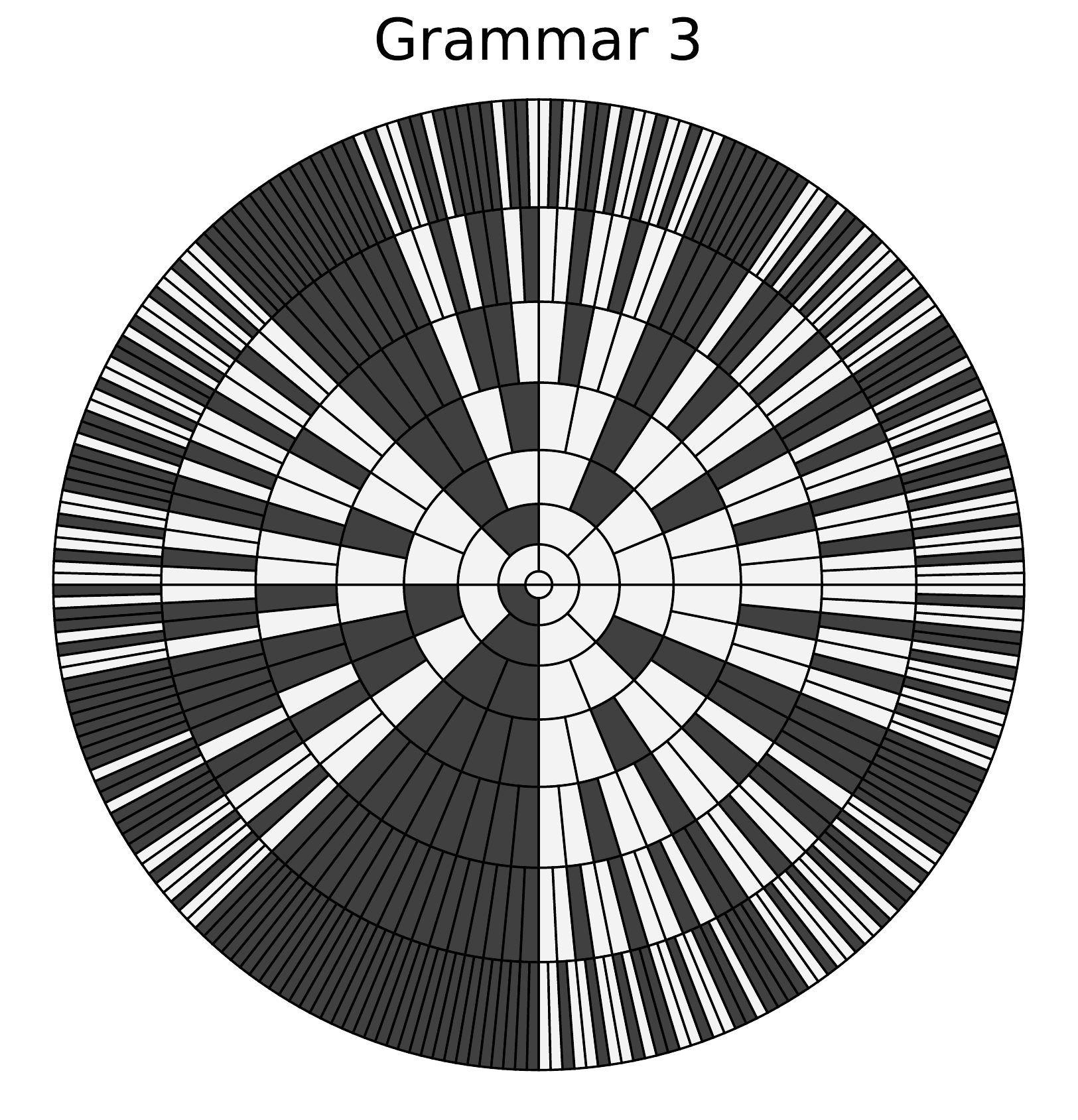}
  \label{fig:g4_pie}
\end{subfigure} \hfill
\begin{subfigure}{.3\textwidth}
  \centering
  \includegraphics[width=.7\linewidth]{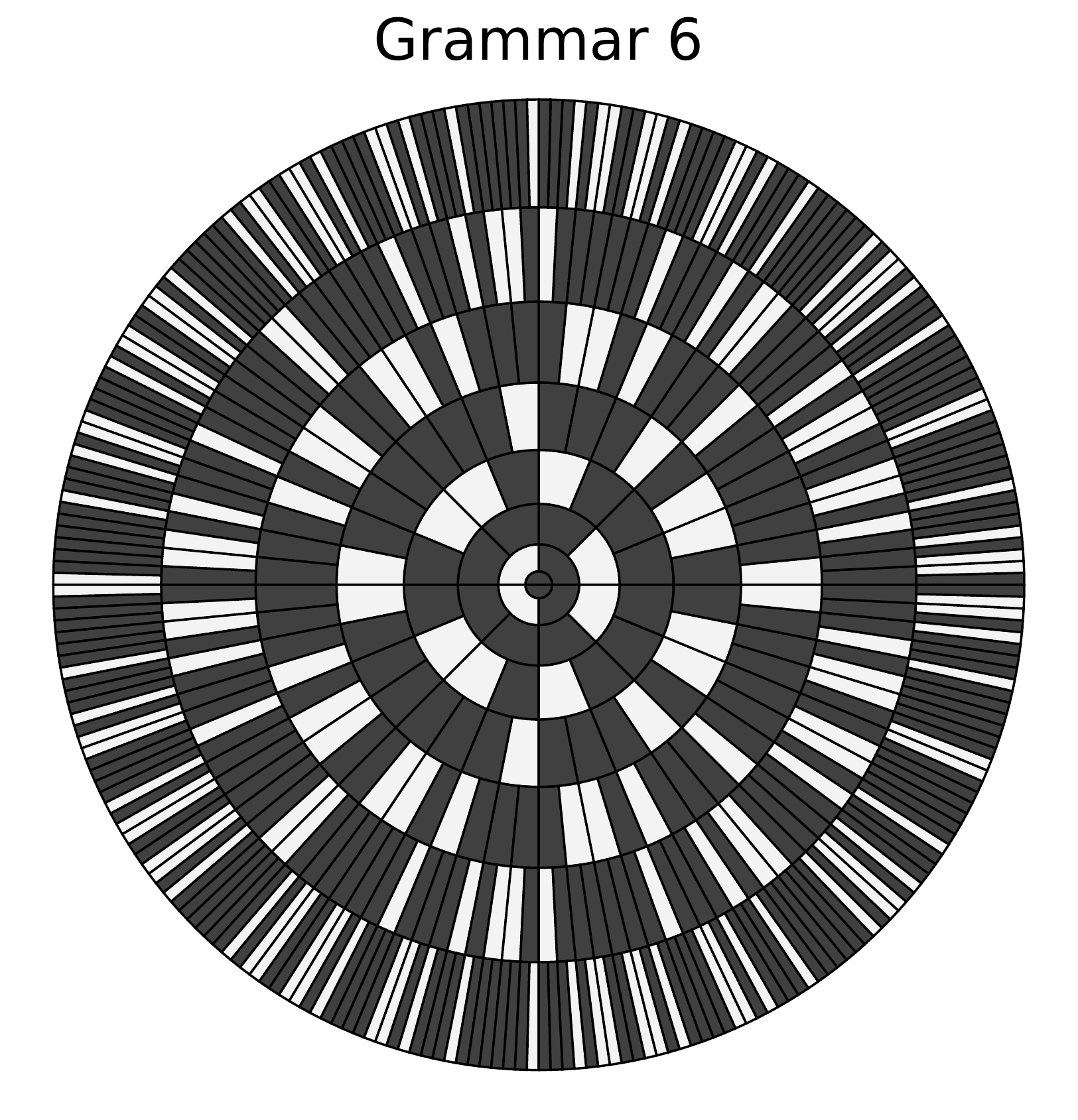}
  \label{fig:g5_pie}
\end{subfigure} \hfill
\caption{Graphical presentation of the distribution of strings of length $N$ ($1 \leq N \leq 8$) for grammars 1, 3 and 6. Each concentric ring of a graph has $2^N$ strings arranged in lexicographic order, starting at $\theta = 0$. 
White and black areas represent accepted and rejected strings respectively.}
\label{fig:tomita_pie}
\end{figure*}

\begin{table*}[t]
\small
\centering
\caption{Hidden update of RNNs selected; building blocks for developing many complicated models. Let $W_*$, $U_*$, and $V_*$ denote weights designed for connecting different neurons, and $b$ denote the bias. $N_f$ is predefined in M-RNN. $\odot$ is the Hadamard product.}
\label{tab:rnn_models}
\begin{tabular}{ll}
\hline\hline
Model & Hidden layer update ($U_{*} \in \mathbb{R}^{N_h \times N_x}, V_{*} \in \mathbb{R}^{N_h \times N_h}, b \in  \mathbb{R}^{N_h \times 1}$) \\ \hline\hline
SRN~\cite{elman1990finding}  & $h^{t} = \phi(U x^t + V h^{t-1} + b)$\\\hline
MI-RNN~\cite{wu2016multiplicative} & \begin{tabular}[c]{@{}l@{}} $h^{t} = {\tt Tanh} (\alpha \odot U x^{t} \odot V h^{t-1} + \beta_{1} \odot V h^{t-1} + \beta_{2} \odot U x^t + b)$  \\
$\alpha \in \mathbb{R}^{N_h}, \beta_1 \in \mathbb{R}^{N_h}, \beta_2 \in \mathbb{R}^{N_h}$ \end{tabular} \\ \hline
M-RNN~\cite{SutskeverMH11} & \begin{tabular}[c]{@{}l@{}} $h^{t} = {\tt Tanh} (W_{hf} \cdot (\mathrm{diag}(W_{fx} x^t) \cdot W_{fh} h^{t-1}) + W_{hx} x^t + b)$  \\
$W_{hf} \in \mathbb{R}^{N_h \times N_f}, W_{fx} \in \mathbb{R}^{N_f \times N_x}$, $W_{fh} \in \mathbb{R}^{N_f \times N_h}, W_{hx} \in \mathbb{R}^{N_h \times N_x}$ \end{tabular} \\ \hline
2-RNN~\cite{giles1992learning}    & $h_{i}^{t} = \phi(\sum_{j,k} W_{kij} h_{j}^{t-1} x_{k}^{t} + b_i), \; i,j = 1 \dots N_h, \; k = 1 \dots N_x$, $W \in \mathbb{R}^{N_h \times N_h \times N_x}$  \\ \hline
LSTM~\cite{hochreiter1997long} & \begin{tabular}[c]{@{}l@{}}$s^{t} = \phi(U_s x^t + V_s h^{t-1})$, $s = \{i, f, o, g \}$ and $\phi = \{ \tt Sigmoid, \tt Tanh \}$\\ 
$c^{t} = c^{t-1} \odot f^{t} + g^{t} \odot i^{t}$, $h^{t} = {\tt Tanh}(c_{t}) \odot o^{t}$ \end{tabular} \\ \hline
GRU~\cite{cho2014properties} & \begin{tabular}[c]{@{}l@{}}$z^{t} = \sigma(U_z x^{t} + V_z h^{t-1})$,  $r^{t} = \sigma(U_r x^t + V_r h^{t-1})$, \\ 
$g^{t} = {\tt Tanh} (U_h x^t + V_h (h^{t-1} \odot r^{t}))$, $h^{t} = (1-z^{t}) \odot g^{t} + z^{t} \odot h^{t-1}$
\end{tabular}  \\ \hline\hline
\end{tabular}
\vspace{-1.0em}
\end{table*}

\subsection{Entropy of a Regular Language from a Data-Driven Perspective}
\label{sec:entropy_data}
The Tomita grammars~\cite{tomita1982learning} define a family of seven relatively simple RGs (shown in Table~\ref{tab:data}), and have been widely used for grammatical inference tasks~\cite{de2010grammatical,watrous1992induction,wang2018empirical,Kan15Kernel,weiss2017extracting}. 
Despite being relatively simple, they represent RGs with a wide range of complexity. 
In Figure~\ref{fig:tomita_pie}, we plot three graphs~\footnote{We follow prior work~\cite{watrous1992induction} to plot these graphs.} for grammars 1, 3, and 6 to illustrate their differences. 
In each graph, every concentric ring contains the sets of strings (with a certain length) accepted and rejected by the corresponding RG. 
Note that the percentages of accepted (or rejected) strings for different grammars are very different. 
For example, on grammars 3 and 6, the numbers of accepted strings are much larger than that of grammar 1. 
This difference is also implied in prior empirical work~\cite{wang2018empirical,weiss2017extracting}, where grammar 6 is much harder to learn than grammars 1 and 3. 
An intuitive explanation is that for grammar 6, flipping any 0 to 1 or vice versa, any accepted or rejected string can be converted into a string with the opposite label. 
A RNN needs to handle such subtle changes in order to correctly recognize all strings accepted by grammar 6. 
Since this change can happen to any digit, a RNN must account for all digits. 

We now formally show that a RG that generates a more balanced set of accepted and rejected strings has a higher level of complexity and seems more difficult to learn. 
Given an alphabet $\Sigma = \{0 ,1\}$, we denote the collection of all $2^N$ strings of symbols from $\Sigma$ with length $N$ as $X^{N}$. 
For a grammar $G$, let $m_p^G$ ($r_p^G$) and $m_n^G$ ($r_n^G$) be the numbers (ratios) of positive and negative strings respectively. 
Assuming that all strings in $X^{N}$ are randomly distributed, we then denote the expected times of occurrence for an event $F_N$ -- two consecutive strings having different labels -- by $\mathrm{E} [F_N]$. 
We have the following definition of entropy for RGs with a binary alphabet.
\begin{definition}[Entropy]
\label{def:entropy}
Given a grammar $G$ with alphabet $\Sigma = \{0 ,1\}$, its entropy is:
\begin{equation}
  \begin{aligned}
  H(G) = \underset{N \rightarrow \infty }{\mathrm{lim\,sup}} \,H^{N}(G)= \underset{N \rightarrow \infty }{\mathrm{lim\,sup}} \,\frac{1}{N}\, \log_{2}\, \mathrm{E} [F_N],      
  \end{aligned}
  \label{eq:entropy}
\end{equation}
where $H^N(G)$ is the entropy calculated for strings with the length of $N$~\footnote{Here we use $\limsup$  to cover certain particular cases, for instance when $N$ is set to odd value for grammar 5.}.
\end{definition}

Furthermore, see the following proposition:
\begin{prop}
\label{prop:entropy}
	\begin{equation}
H(G)= 1 + \underset{N \rightarrow \infty }{\mathrm{lim\,sup}} \,\frac{\log_{2} \big(r_p^G(1 - r_p^G)\big)}{N}.
\label{eq:entropy_calculation}
	\end{equation}
\end{prop}

\begin{proof}
Given any concentric ring (corresponding to the set of strings with a length of $N$) shown in Figure 1, let $R$ denote the number of consecutive runs of strings, and $R_p$ and $R_n$ denote the number of consecutive runs of positive strings and negative strings in this concentric ring respectively. 
Then we have $ \mathrm{E}[F]=\mathrm{E}[R] -1 = \mathrm{E}[R_p] + \mathrm{E}[R_n] - 1 $. 
Without loss of generality, we can choose the first position as $\theta = 0$ in the concentric ring. Then we introduce an indicator function $I$ by $I_i = 1$ representing that a run of positive strings starts at the $i$-th position and $I_i = 0$ otherwise. 
Since $ R_p = \sum_{i=1}^{2^N} I_i$, we have 
\begin{equation}
\vspace{-0.1em}
  \begin{aligned}
  \mathrm{E} [R_p] \!=\! \sum_{i=1}^{2^N} \mathrm{E} [I_i] \;\; \text{and} \;\; \mathrm{E} [I_i] \!=\! \begin{cases} m_{p}^G/{2^N}  , \qquad \qquad  \quad \; i = 1\\
 m_n^G m_p^G / 2^N (2^N-1) , \, i \neq 1. \nonumber
  \end{cases}
  \end{aligned}
  \vspace{-0.1em}
\label{eq:E_runs}
\end{equation}
As such, we have 
\begin{equation}
\vspace{-0.1em}
  \begin{aligned}
  \mathrm{E} [R_p] = \frac{m_{p}^G(1 + m_{n}^G)}{2^N} \;\; \text{and} \;\; \mathrm{E}[R_n] = \frac{m_{n}^G(1 + m_{p}^G)}{2^N}. \nonumber
  \end{aligned}
\label{eq:E_runs_p}
\end{equation}
By substituting $\mathrm{E}[F]$ into the entropy definition, we have
\begin{equation}
  \begin{aligned}
  H(G) = 1 + \underset{N \rightarrow \infty }{\mathrm{lim\,sup}} \,\frac{\log_{2} \big(r_p^G(1 - r_p^G)\big)}{N}.
  \end{aligned}
\end{equation}

\end{proof}

Proposition~\ref{prop:entropy} implies that a RG generating more balanced string sets has a higher entropy value. 
As such, with the following theorem, we can categorizes all RGs with a binary alphabet based on their entropy values.

\begin{thm} Given any regular grammar $G$ with $\Sigma = \{0 ,1\}$, it belongs to one of following classes:\\
(a)~\emph{Polynomial} class. $H(G) \!=\! 0$, iff $m_p^G \sim P(N)$, where $P(N)$ denotes the polynomial function of $N$; \\ 
(b)~\emph{Exponential} class. $H(G) = \log_{2} b \in (0,1)$, iff $m_p^G \sim \beta \cdot b^N$ where $b < 2$ and $\beta > 0$; \\
(c)~\emph{Proportional} class. $H(G) \!=\! 1$, iff $m_p^G \sim \alpha \cdot 2^N$, where $\alpha \in [0,1)$.\\
Here $\sim$ indicates that some negligible terms are omitted when $N$ approaches infinity.
\label{thm:entropy}
\vspace{-0.2em}
\end{thm}

\begin{proof}
For each class of grammars, given that their $m_p$ takes the corresponding form shown in Theorem 1, the proof for the sufficient condition is trivial and can be checked by applying \emph{L'Hospital's Rule}. As such, in the following we only provide a proof for the necessary condition. 
From~\eqref{eq:entropy_calculation}, we have:
\begin{equation}
\vspace{-0.5em}
	\begin{aligned}
	\nonumber 
	H(G) &= \lim_{N \to \infty}\frac{\log_2(m_p \cdot 2^N-m_p^2)}{N}-1 \\
         &= \lim_{N \to \infty}\frac{m_p'\cdot 2^N+\ln2\cdot 2^N \cdot m_p-2m_p \cdot m_p'}{\ln2 \cdot (m_p \cdot 2^N-m_p^2)}-1 \\
         &= \lim_{N \to \infty}\frac{m_p' \cdot 2^N+\ln2\cdot m_p^2 - 2m_p \cdot m_p'}{\ln2 \cdot m_p \cdot (2^N-m_p)},
	\end{aligned}
	\vspace{-0.5em}
\end{equation}
where $m_p'$ denotes the derivative of $m_p$ with respect to $N$. It is easy to check that $\lim_{N \to \infty}\frac{m_p'}{m_p}$ exists for regular grammars, then we separate the above equation as follows:
\begin{equation}
\vspace{-0.5em}
	\begin{aligned}
	\nonumber 
	H(G) = \lim_{N \to \infty}\frac{m_p'}{\ln2\cdot m_p} + 
	       \lim_{N \to \infty}\frac{1-\frac{m_p'}{\ln2\cdot m_p}}{\frac{2^N}{m_p}-1}.
    \end{aligned}
    \vspace{-0.1em}
\end{equation}
It should be noted that the second term in the above equation equals $0$. Specifically, assuming that $m_p$ has the form of $\alpha \cdot b^N$ where $b < 2$ ($b$ cannot be larger than 2 for binary alphabet), then the denominator of the second term is infinity. If $m_p$ has the form of $\alpha\cdot{2}^N$, then the numerator tends to zero while the denominator is finite. As such, we have
\begin{equation}
	\begin{aligned}
\nonumber H(G) = \lim_{N \to \infty}\frac{m_p'}{\ln2\cdot m_p}.
    \end{aligned}
    \vspace{-0.5em}
\end{equation}

If $H(G)=0$, then we have $\lim_{N \to \infty}\frac{m_p'}{m_p}=0$, indicating that the dominant part of $m_p$ has a polynomial form of $N$ hence $m_p \sim P(N)$, where $P(N)$ denotes the polynomial function of $N$. 
If $H(G)=t\neq 0$, then we have $\lim_{N \to \infty}\frac{\ln(m_p)}{tN\ln2}=1$, which gives that $m_p \sim  \beta \cdot{2}^{tN}$, where $\beta > 0$. If $t = \log_2 b$, then we have $m_p$ $\sim \beta \cdot{b}^{N}$ where $b < 2$. Furthermore, if $t=1$, we have $m_p \sim \alpha\cdot{2}^{N}$ where $\alpha \in [0,1)$.

\end{proof}

For Tomita grammars, we categorize grammars 1, 2, and 7 into the polynomial class, grammars 3, 4 into the exponential class and grammars 5, 6 into the proportional class according to their entropy values. 
When compared to the entropy in shift space, which only considers accepted strings, our Definition~\ref{def:entropy} considers both the accepted and rejected strings. 
This is naturally more informative and leads to benefits in various dimensions. 
For example, given a data set with samples reasonably sampled from an unknown data set, then we can estimate the complexity of this unknown data set by calculating the entropy of the available data set. 
Also, for a \emph{k}-class classification task with strings of length $N$, let $m_{i}$ denote the number of strings in the $i$th class. 
Then we have $\mathrm{E} [F_N] = 2^N -\frac{1}{2^N}\cdot \sum_{i=1}^{k}  m_{i}^{2}$. 
We can then easily generalize Definition~\ref{def:entropy} to a \emph{k}-class classification case by substituting this in Definition~\ref{def:entropy}. 
However, this can be challenging for the entropy defined for the shift space since it can only be constructed in a one-versus-all manner. 
Also, the shift space cannot express all RGs, especially for grammars that lack shift-invariant and closure properties~\cite{lind1995introduction}. 

\subsection{The Entropy of Regular Language from a DFA Perspective}
\label{sec:entropy_dfa}
Here we provide an alternative way to obtain the entropy of a RG using the transition matrices of its associated minimal DFA. 
This approach can be directly applied and provide immediate results if the minimal DFA is available. 
As such, we can alleviate the computation cost of the data-driven approach. 
Here we again mainly illustrate the case when the alphabet size is two, and the extension for grammars with larger alphabets are provided in the Appendix.

\begin{thm} 
Given a grammar $G$ with the alphabet $\Sigma = \{0, 1\}$ and its associated minimal DFA $M$ with $n$ states, let $T_0$, $T_1 \in \mathbb{Z}^{n \times n}$ denote the transition matrices of $M$ associated with input $0$ and $1$, and have $T=T_0+T_1$. We use $k(T)$ to represent the number of diagonal elements equal to $2$ and $\sigma(T)$ to denote the set of modulus of all eigenvalues of $T$. Then $G$ belongs to one of following classes:\\
(a)~\emph{Polynomial} class. $H(G) \!=\! 0$ iif $k(T)=1$ and $\sigma(T)=\{1,2\}$; \\
(b)~\emph{Exponential} class. $H(G) = \log_2 \left | \lambda _2 \right |  \in (0,1)$ iif
$k(T)=1$ and $\sigma(T)-\{1,2\} \neq \emptyset$, where $\left | \lambda _2 \right |$ denotes the second largest modulus of the eigenvalues of $T$; \\
(c)~\emph{Proportional} class. $H(G) \!=\! 1$ iif $k(T)=0$ or $k(T)=2$.
\label{thm:spectrum}
\end{thm}

Theorem~\ref{thm:spectrum} indicates that the entropy of a RG lies in the spectrum of its associated DFA. 
Specifically, in the polynomial and exponential classes, a DFA with its summed transition matrix $T$ having only one diagonal element that is equal to 2 indicates that this DFA has only one ``absorbing'' state (either the accepting or rejecting state). Assume that a DFA has one absorbing state and is running over a string. 
Once reaching the absorbing state, this DFA makes a permanent decision -- either acceptance or rejection -- on this string, regardless of the ending symbol has been read or not. 
In comparison, in the proportional class, a DFA can have either zero or two absorbing states (one accepting state and one rejecting state). 
In the case of the Tomita grammars, every grammar has exactly one absorbing state except for grammar 5 and 6, which have no absorbing state. 
The DFAs for grammar 5 and 6 can only determine the label of a string after processing the entire string.

%% file: base/4_rnn.tex
\section{The Correspondence Between RNNs and DFAs}
\label{sec:connection}

In this section, 
we show the equivalence between the linear 2-RNN and DFA~\footnote{Recent work~\cite{rabusseau2018connecting} proves a linear 2-RNN equivalent to weighted automata which generalizes all DFA.}. 
Then we examine if other RNNs can learn DFAs from all or certain classes previously introduced. 
Here we only consider RNNs with linear hidden activation for analytical convenience. 
In Section~\ref{sec:eval}, we empirically validate the analysis in this section with RNNs configured with nonlinear activations.  

\begin{table*}[!h]
\centering
\small
\caption{Solutions for approximating DFA with 2-RNN, Elman-RNN, and MI-RNN.}
\label{tab:solutions}
\begin{tabular}{|l|l|l|l|}
\hline \hline
RNN (Linear) & RG  & Solutions & Configuration \\ \hline \hline
2-RNN & All & $W_i = T_i, i=1,\dots I$ & \begin{tabular}[l]{@{}l@{}}$W^{\prime}=W,U^{\prime}=0$,\\$V^{\prime}=0,b^{\prime}=0$\end{tabular} \\ \hline

SRN & \begin{tabular}[l]{@{}l@{}}Grammars accept strings\\with certain length. \end{tabular} & \begin{tabular}[l]{@{}l@{}}$V=\sum_{i=1}^{I} \alpha_i T_i$, $c_i=0$,\\$\sum_{i=1}^{I} \alpha_i =1$, $\alpha_i >0$ \end{tabular} & \begin{tabular}[l]{@{}l@{}}$W^{\prime}=0,U^{\prime}=U$,\\$V^{\prime}=V,b^{\prime}=b$\end{tabular} \\ \hline

MI-RNN & \begin{tabular}[l]{@{}l@{}}Grammars accept strings\\with certain length.\end{tabular} & \begin{tabular}[l]{@{}l@{}} $V=\sum_{i=1}^{I} \alpha_i T_i$, $U=\mathbf{1}^{n \times I}$,\\$\sum_{i=1}^{I} \alpha_i =1$, $\alpha_i >0$ \end{tabular} & \begin{tabular}[l]{@{}l@{}} ${W^{\prime}}_{ijk}=\alpha U_{ji} V_{jk}$, $b^{\prime}=b$,\\$U^{\prime}=\beta_2 U$, $V^{\prime}=\beta_1 V$ \end{tabular}\\ \hline \hline
\end{tabular}
\vspace{-1.0em}
\end{table*}

\subsection{Linear 2-RNN and DFA Relationship}
Given a DFA with an $I$-size alphabet $\Sigma_I$ and a minimal number of $n$ states, we denote the transition matrix for each input symbol as $T_i \in \mathbb{Z}^{n \times n}$. 
Each column of $T_i$ sums to 1. 
Given an input symbol $i \in \Sigma_I $, the DFA state transition is $h^{t} = T_i h^{t-1}$, where $h^t$ denotes the hidden vector at time $t$ in the $n$-dimensional unit cube $\mathcal{H}=[-1,1]^n$. 
Assume the input of a linear 2-RNN is one-hot encoded, then a linear 2-RNN can be constructed that exactly matches the DFA by solving:
\begin{equation}
  \underset{W_i \in \mathbb{R}^{n \times n}}{\mathrm{min}} \int_{\mathcal{H}} \sum_{i=1}^{I}\left | T_i h  - W_i h \right | \mathrm{d} \mu(h).
  \label{eq:2rnn}
\end{equation}
The optimum is obtained when $W_i = T_i$ for $i = 1, \dots, I$. 
Although this optimum is challenging to reach when $T_i$ is not available, \eqref{eq:2rnn} indicates that a 2-RNN can be stably constructed to resemble any DFA~\cite{omlin1996stable,CarrascoFVN00}. 
As for M-RNN, similar results can be obtained if its adopted decomposition retains a significant fraction of the tensor's expressive power~\cite{SutskeverMH11}.

\subsection{Correspondence Between SRN, MI-RNN and DFA}
Here we first fit the reformulated linear SRN and linear MI-RNN into the optimization framework introduced for linear 2-RNN. Specifically, for SRN, we use $c_i = U x_t + b$ (shown in Table~\ref{tab:rnn_models}) to represent the input-dependent term. We assume $h$ is uniformly distributed in $\mathcal{H}$. Then we have:
\begin{equation}
  \underset{V \in \mathbb{R}^{n \times n}, c_i \in \mathbb{R}^{n \times 1}}{\mathrm{min}} \int_{\mathcal{H} } \sum_{i=1}^{I}\left | T_i h - V h + c_i \right | \mathrm{d} h.
  \label{eq:elman}
\end{equation}
For MI-RNN, we only consider the $h^{t+1}=U x^{t} \odot V h^{t}$ term which dominates the transition and have
\begin{equation}
  \underset{V \in \mathbb{R}^{n \times n}, U_i \in \mathbb{R}^{n \times 1}}{\mathrm{min}} \int_{\mathcal{H}} \sum_{i=1}^{I}\left | T_i  - ((\mathbf{1_n} \otimes U_i^{\mathrm{T}})\odot V) \right | \mathrm{d} h,
\label{eq:mirnn}
\end{equation}
where $\mathbf{1_n}$ is the all 1 column vector, $U_i$ is $i$-th column of $U$. 
$\otimes$ denotes the Kronecker product. 
The solutions for~\eqref{eq:elman} and~\eqref{eq:mirnn} are shown in the third column in Table~\ref{tab:solutions}. 
From~\eqref{eq:elman}, it is easy to check given a fixed $V$, the optimum can be achieved when $T_i = T_j$ for $i,j = 1,\dots, I$ and $c_i = 0$. 
This indicates that SRN can only accurately learn DFAs that recognizing strings with a certain length, therefore is limited in its capability of approximating all three classes of DFAs. 
Specifically, a DFA that only accepts strings with certain length can either belong to the polynomial or exponential class. A similar result can be obtained for MI-RNN and will be discussed in the following.

\subsection{An Unified View of Different RNNs}
While linear 2-RNN is better at modeling RGs than MI-RNN and SRN, these latter models can be more suitable for modeling other types of sequential data. 
For example, prior work~\cite{connor1991recurrent,connor1992recurrent} shows that SRN generalizes nonlinear autoregressive-moving average models. 
As such, we now propose a unified framework that integrates different orders of hidden interaction, while preserving the advantages of different RNNs:
\begin{equation}
  h^{t}=\phi(\sum_i^I W^{\prime}_i x_i^t h^{t-1}+U^{\prime} x^t + V^{\prime} h^{t-1}+b^{\prime}),
  \label{eq:unified}
\end{equation}
where $W^{\prime}\in \mathbb{R}^{I\times n\times n}$. 
We show in the fourth column of Table~\ref{tab:solutions} about how to configure this unified RNN (UNI-RNN) to express SRN, MI-RNN, and 2-RNN. 
Specifically, for MI-RNN, we can see that in its associated unified framework, ${W^{\prime}}_{ijk} {W^{\prime}}_{njm} = {W^{\prime}}_{ijm} {W^{\prime}}_{njk}$. 
This indicates that for a pair of input symbols, their sequence order can be switched to reach the same state. 
This form of transition matrix corresponds to a DFA that only accepts strings with a certain length. 
This unified framework applies flexible control of the hidden layer of a RNN during the learning process. 
Specifically, the four terms on the right-hand side of~\eqref{eq:unified} represent the input-dependent rotation and translation, input-independent rotation and translation, respectively. 
Other RNNs only partially support these transformations. 
As such, we expect the UNI-RNN to be more flexible for modeling different types of sequential data.


%% file: base/5_eval.tex
\section{Evaluation}
\label{sec:eval}

\begin{figure*}[t]
\centering
\begin{subfigure}{.32\textwidth}
  \centering
  \includegraphics[width=0.95\linewidth]{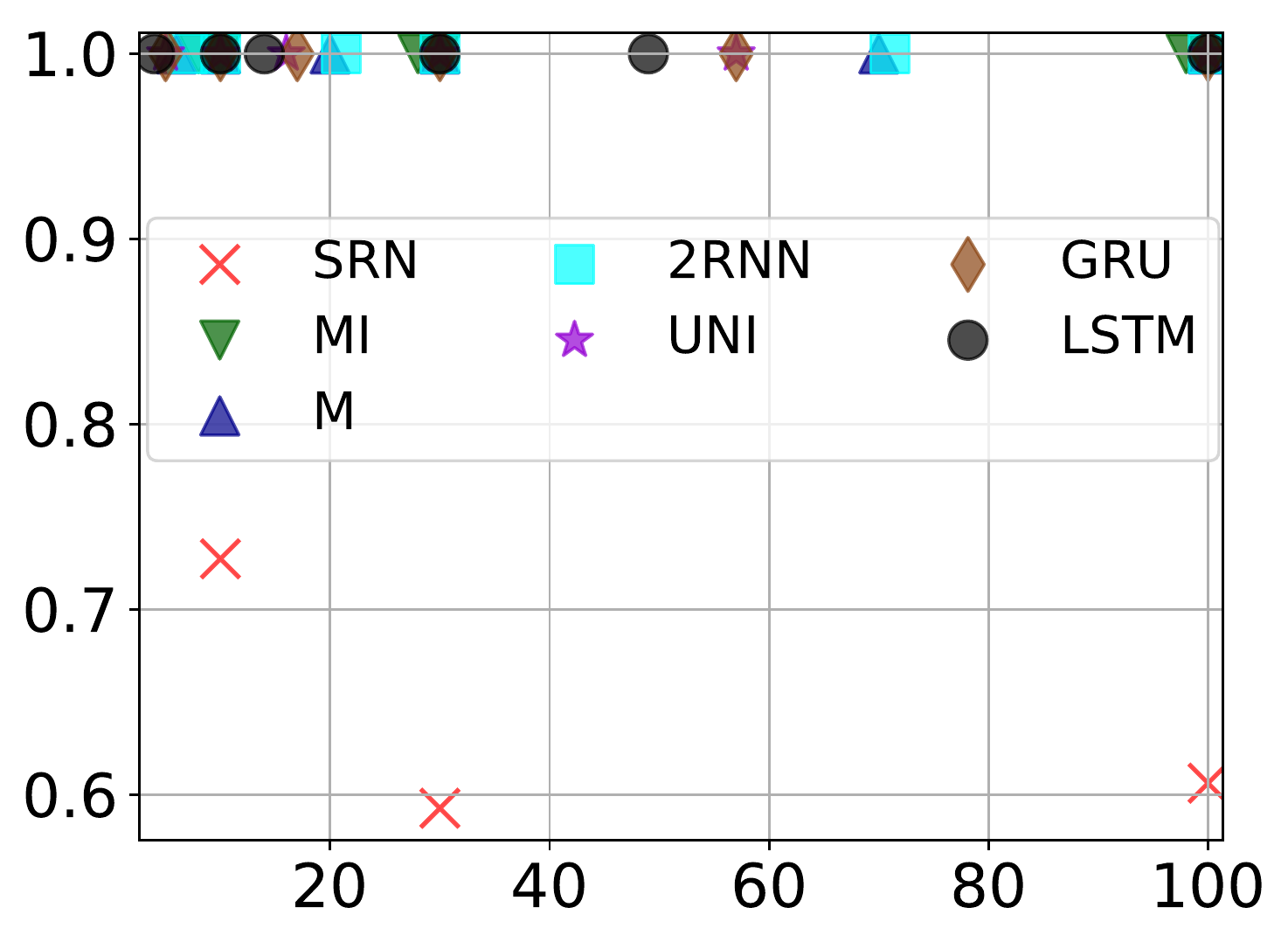}
  \centering
  \vspace{-0.5em}
  \caption{Grammar 1 (G1).}
  \label{fig:g1_scatter}
\end{subfigure} \hfill
\begin{subfigure}{.34\textwidth}
  \centering
  \includegraphics[width=0.95\linewidth]{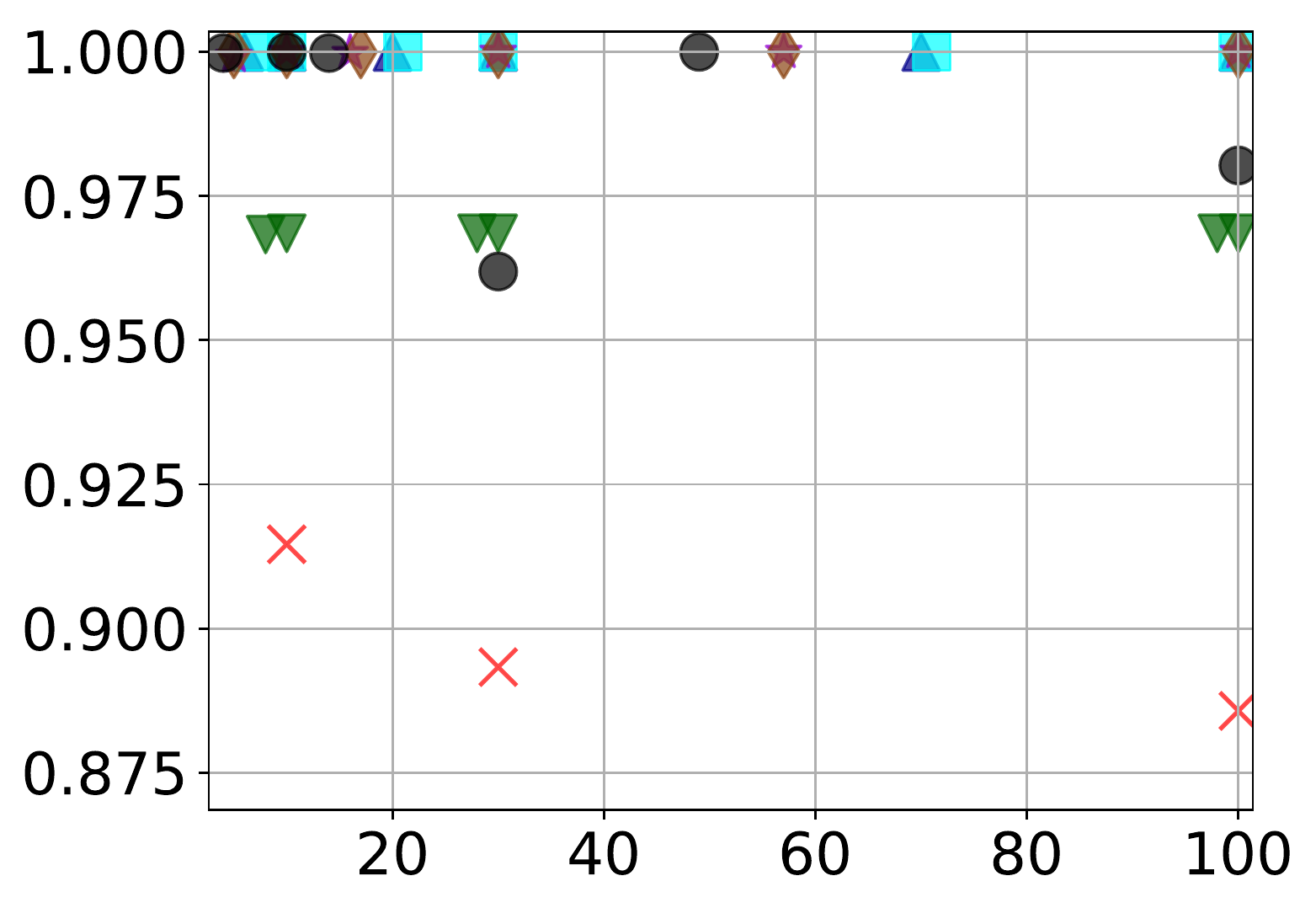}
  \vspace{-0.5em}
  \caption{Grammar 3 (G3).}
  \label{fig:g3_scatter}
\end{subfigure} \hfill 
\begin{subfigure}{.32\textwidth}
  \centering
  \includegraphics[width=0.95\linewidth]{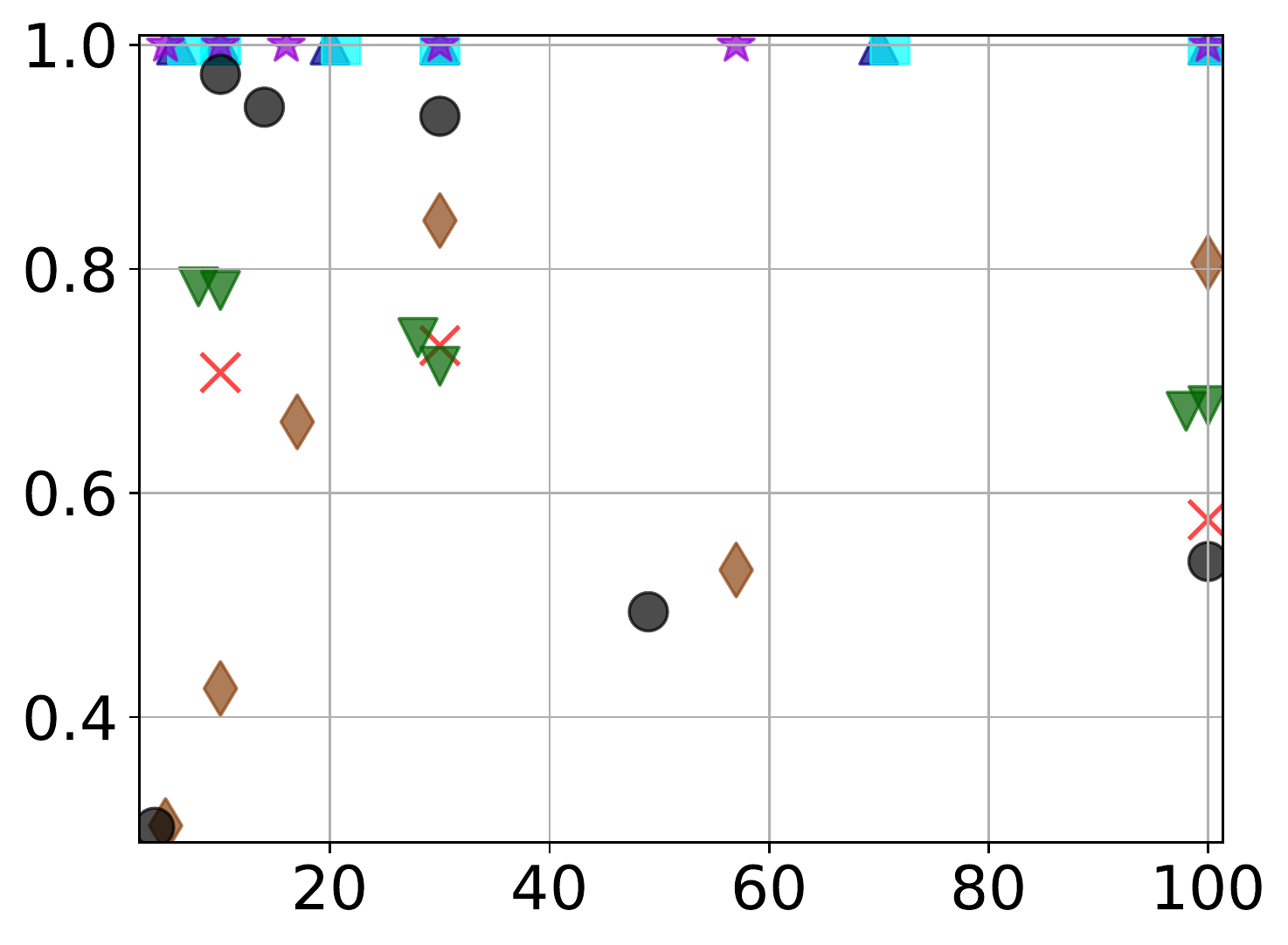}
  \vspace{-0.5em}
  \caption{Grammar 6 (G6).}
  \label{fig:g6_scatter}
\end{subfigure} \hfill \\
\begin{subfigure}{.22\textwidth}
  \centering
  \includegraphics[width=0.95\linewidth]{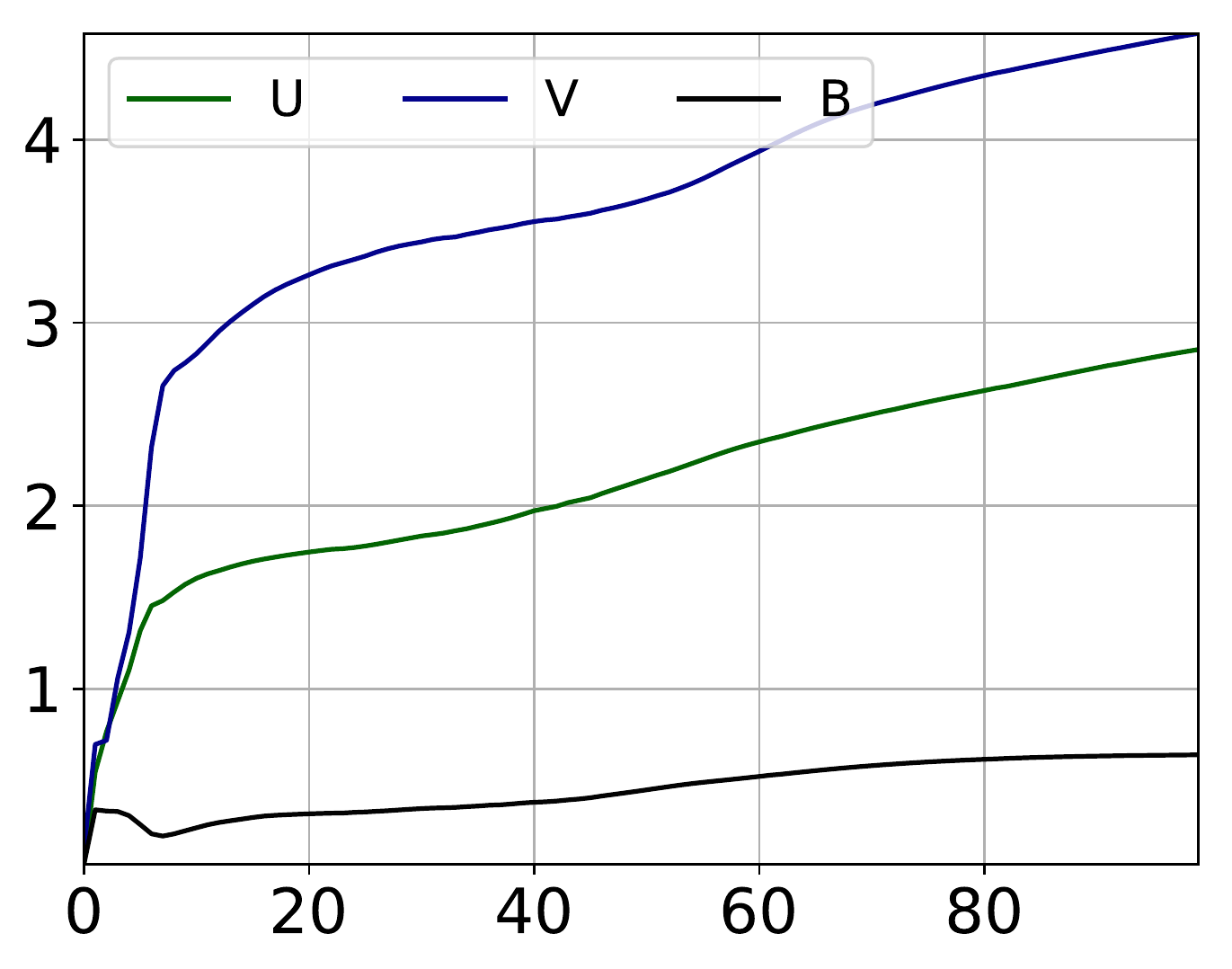}
  \centering
  \vspace{-0.5em}
  \caption{G3 - SRN.}
  \label{fig:g3_srn_line}
\end{subfigure} \hfill
\begin{subfigure}{.22\textwidth}
  \centering
  \includegraphics[width=0.95\linewidth]{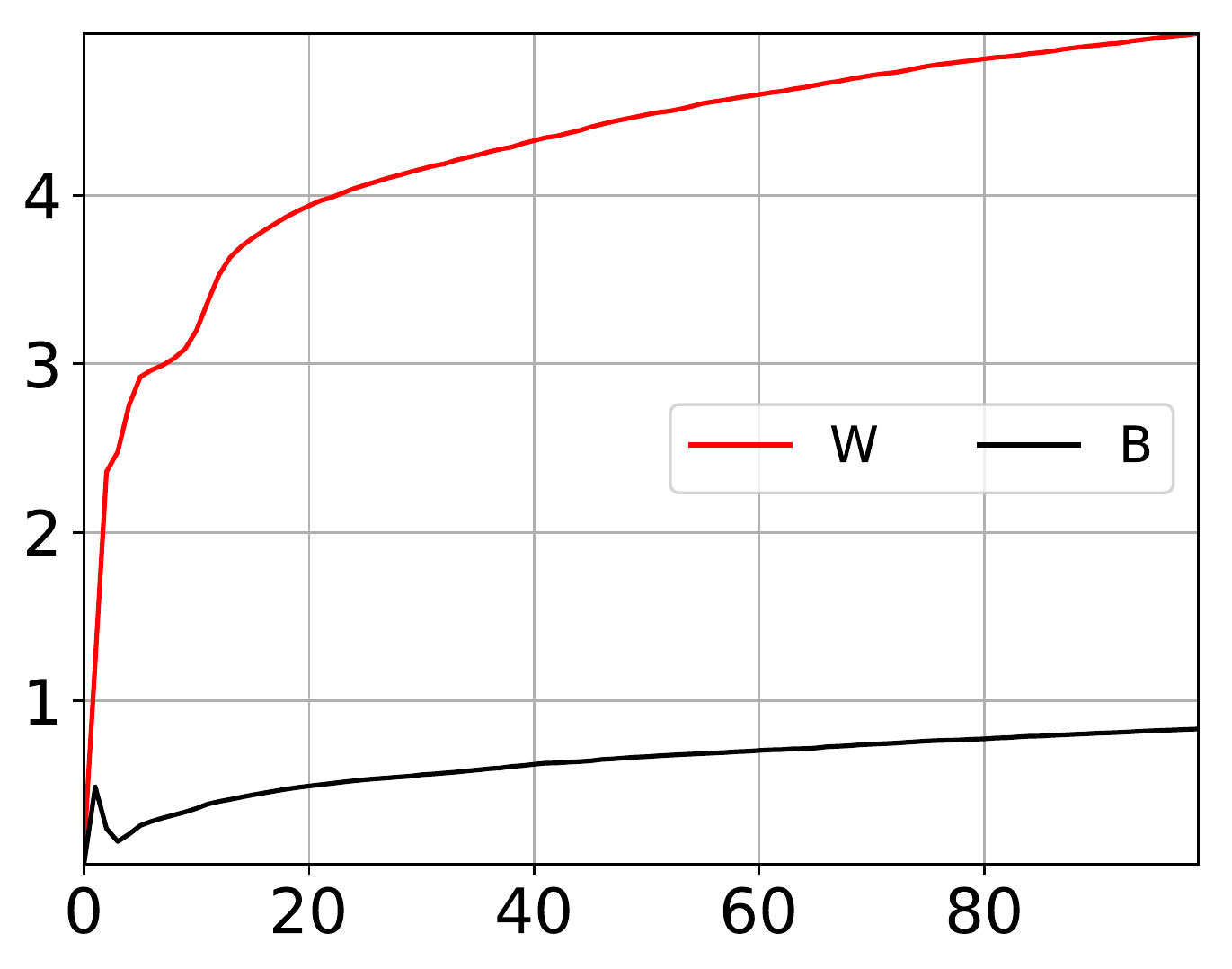}
  \vspace{-0.5em}
  \caption{G3 - 2-RNN.}
  \label{fig:g3_o2_line}
\end{subfigure} \hfill 
\begin{subfigure}{.22\textwidth}
  \centering
  \includegraphics[width=0.95\linewidth]{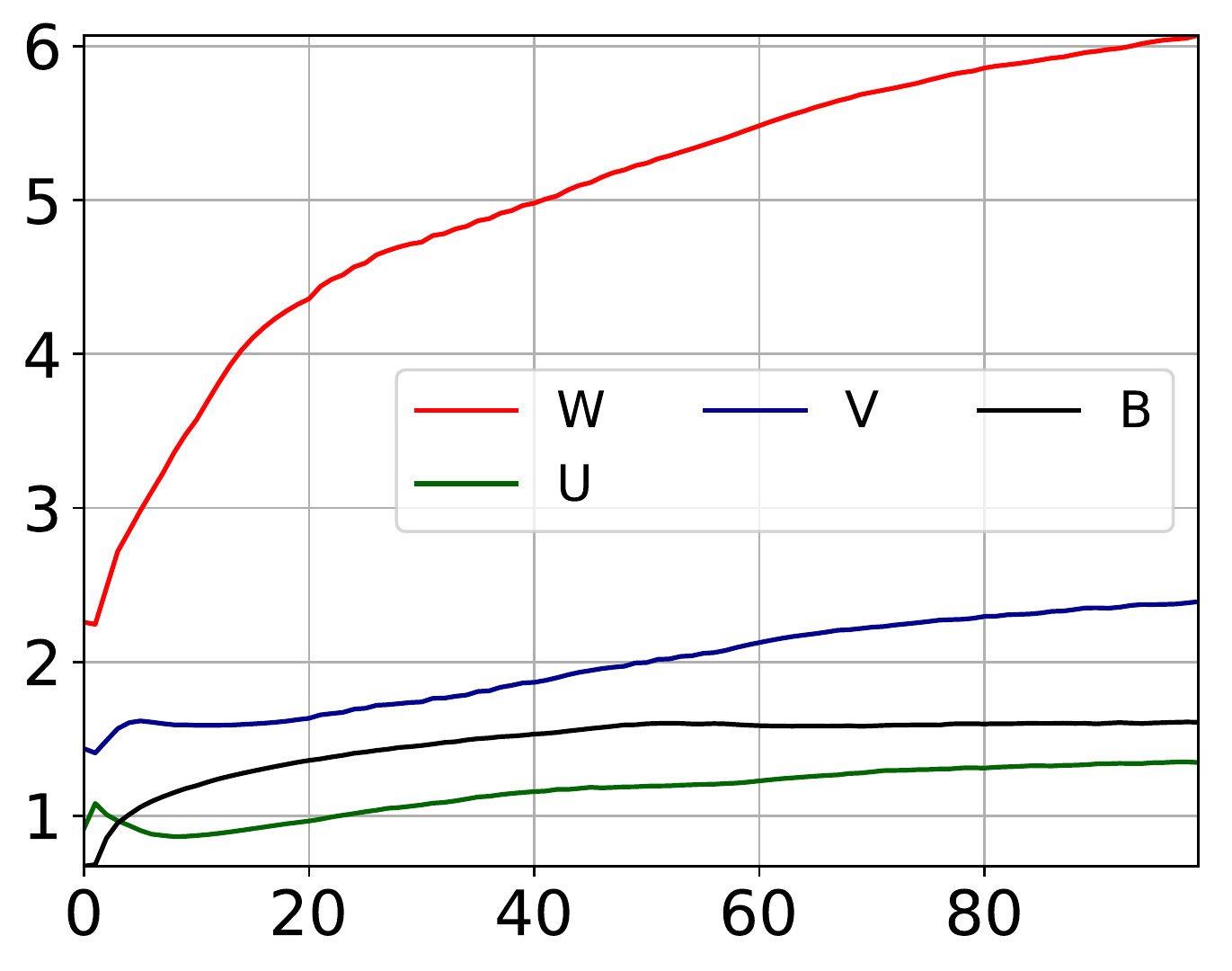}
  \vspace{-0.5em}
  \caption{G3 - UNI-RNN +.}
  \label{fig:g3_uni_p_line}
\end{subfigure} \hfill 
\begin{subfigure}{.22\textwidth}
  \centering
  \includegraphics[width=0.95\linewidth]{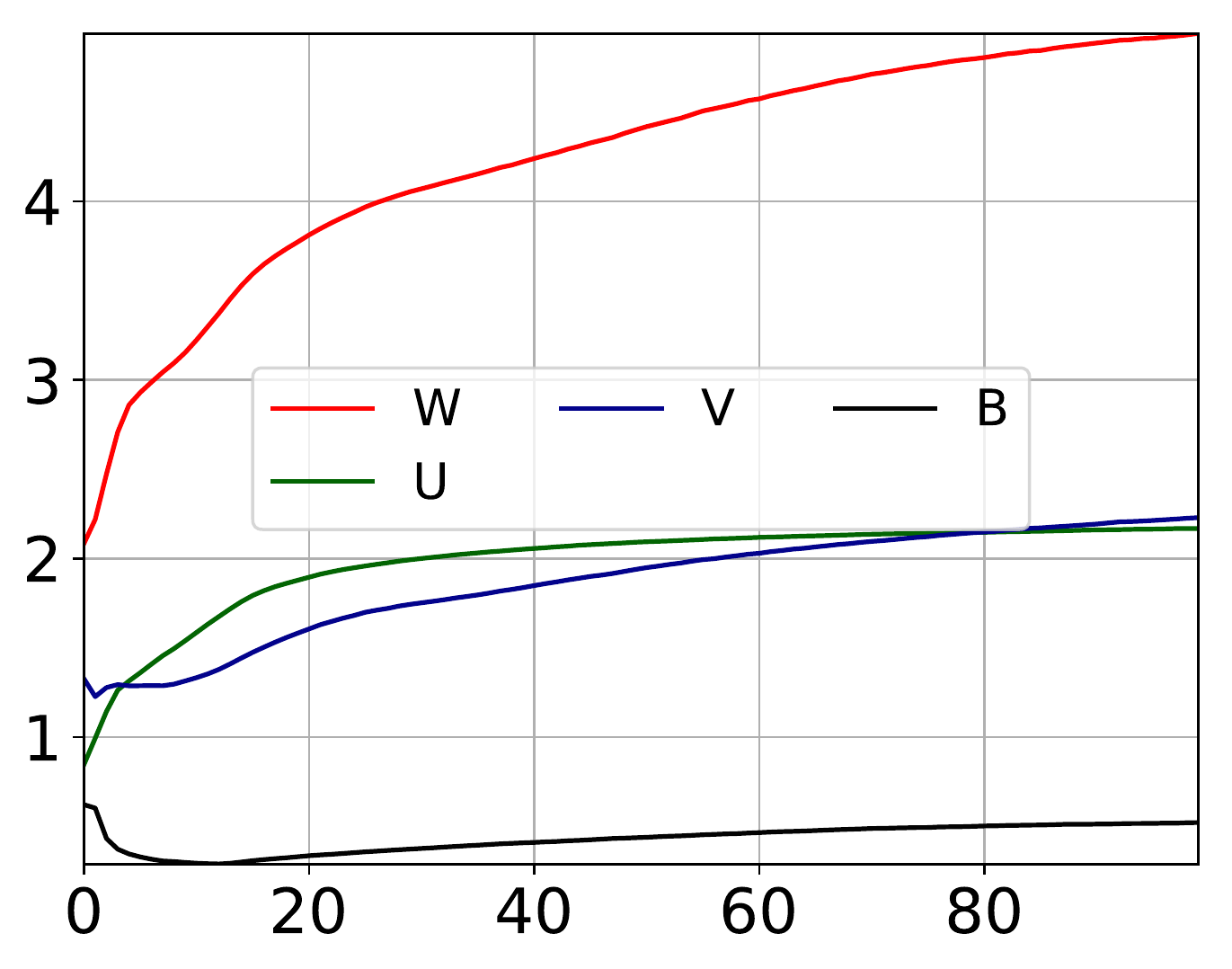}
  \centering
  \vspace{-0.5em}
  \caption{G3 - UNI-RNN -.}
  \label{fig:g3_uni_n_line}
\end{subfigure} \hfill
\caption{Evaluation results on Tomita-1, 3, 6 grammars.}
\label{fig:tomita_scatter}
\vspace{-1.1em}
\end{figure*}

\begin{table*}[t!]
\small
\centering
\caption{Evaluation results. M$_{1 - 7}$ denotes SRN, MI-RNN, M-RNN, 2-RNN, GRU and LSTM.}
\label{tab:results}
\begin{tabular}{|c|c|cccc|c|cccc|c|c|}
\hline \hline
\multirow{2}{*}{RNN} & 
\multicolumn{1}{c|}{\multirow{2}{*}{$N_h$}} & 
\multicolumn{2}{c|}{SL4} & \multicolumn{2}{c|}{SP8} & 
\multicolumn{1}{c|}{\multirow{2}{*}{$N_h$}} & 
\multicolumn{1}{c|}{STAMINA-$\left |\Sigma\right |$-50} &
\multicolumn{1}{c|}{\multirow{2}{*}{$N_h$}} & 
\multicolumn{1}{c|}{\multirow{2}{*}{PTB}} \\ \cline{3-6} \cline{8-8}
 & & T-1 & T-2 & T-1 & T-2 & & \multicolumn{1}{c|}{12.5\% - 25\% - 50\% - 100\%}  & \multicolumn{1}{c|}{} & \multicolumn{1}{c|}{} \\ \hline \hline
M$_1$  & 30 & 0.93  & 0.96  & 0.98  & 0.53  & 100 & \multicolumn{1}{c|}{0.83 -- 0.85 -- 0.91 -- 0.92} & \multicolumn{1}{c|}{650} & \multicolumn{1}{c|}{113.15}\\
M$_2$   & 28 & 0.99  & 0.99  & 0.99  & 0.91  & 98  & \multicolumn{1}{c|}{0.95 -- 0.95 -- 0.95 -- 0.94} & \multicolumn{1}{c|}{450} & \multicolumn{1}{c|}{90.67}\\
M$_3$    & 20 & 1.00  & 1.00  & 0.99  & 0.94  & 64  & \multicolumn{1}{c|}{0.97 -- 0.91 -- 0.91 -- 0.85} & \multicolumn{1}{c|}{450} & \multicolumn{1}{c|}{91.68}\\
M$_4$    & 21 & 1.00  & 1.00  & 1.00  & 1.00  & 17  & \multicolumn{1}{c|}{0.98 -- 0.98 -- 0.96 -- 0.94} & \multicolumn{1}{c|}{450} & \multicolumn{1}{c|}{93.79}\\
M$_5$  & 16 & \textbf{1.00}  & \textbf{1.00}  & \textbf{1.00}  & \textbf{1.00}  & 16  & \multicolumn{1}{c|}{\textbf{0.98} -- \textbf{0.98} -- \textbf{0.96} -- \textbf{0.95}} & \multicolumn{1}{c|}{400} & \multicolumn{1}{c|}{\textbf{89.70}}\\
M$_6$  & 17 & 1.00  & 1.00  & 1.00  & 0.99  & 50  & \multicolumn{1}{c|}{0.94 -- 0.93 -- 0.91 -- 0.87} & \multicolumn{1}{c|}{400} & \multicolumn{1}{c|}{88.00}\\
M$_7$ & 14 & 1.00  & 1.00  & 0.99  & 0.99  & 41  & \multicolumn{1}{c|}{0.93 -- 0.90 -- 0.87 -- 0.84} & \multicolumn{1}{c|}{400} & \multicolumn{1}{c|}{84.30}\\ \hline \hline
\end{tabular}
\vspace{-1.0em}
\end{table*}

We evaluated and compared all RNNs on string sets generated by different RGs with different levels of complexity and on the PTB data set to explore the merits of different RNNs.~\footnote{All implementations are available at~\url{https://github.com/lazywatch/HighOrderRNN}}

\subsection{Recurrent Networks Setup}
\label{sec:rnn_setup}

To better demonstrate the difference between RNNs, we configured each RNN with the same setting: one-hot encoding for the input, one single hidden layer, and the {\tt Tanh} hidden activation. 
We used SRN as the baseline and configured other RNNs to have either the hidden layer of the same size or the same number of total parameters as SRN. 
In the former case, higher-order RNNs have more parameters than lower-order ones. 
In the latter case, we followed the prior work~\cite{wu2016multiplicative}, which compared SRN and MI-RNN, to ensure that higher-order RNNs have strictly fewer parameters than lower-order ones. 
In particular, we configured MI-RNN as done in the original work, and the extra dimension $N_f$ of M-RNN to be the same as $N_h$. 
GRU and LSTM were configured by only comparing with SRN.
All RNNs were initialized in the same manner of uniformly drawing samples from $[-0.02, 0.02]$, and trained with RMSprop with the learning rate of $0.01$. 
For each RNN, its learning performance was evaluated by either the F1 score or Balanced Classification Rate (BCR) depending on the data sets.

\subsection{The Tomita Grammars}
\label{sec:eval_tomita}
We followed the latest work~\cite{weiss2017extracting} using the Tomita grammars, and its implementation~\footnote{\url{https://github.com/tech-srl/lstar_extraction}} to generate the string sets. 
Specifically, for the training sets, we uniformly sampled strings of various lengths $L \in \{0,\dots,13, 16, 19, 22\}$ for all seven grammars. 
For the testing sets, each set contains up to 1000 uniformly sampled strings of each of the lengths $L \in \{1, 4, 7, \dots, 28\}$. 
All RNNs were trained up to 100 epochs with batch size of 100. 

In Figure~\ref{fig:g1_scatter}-\ref{fig:g6_scatter}, we present the results obtained by all RNNs on the Tomita 1, 3, and 6 grammars as each represents a distinct class mentioned previously. 
The horizontal axis shows the sizes of the hidden layer configured for all RNNs, while the vertical axis shows their obtained F1 scores.  
On each grammar, we varied the sizes of the hidden layer of SRN in $\{10,30,100\}$. 
Then for any other RNN, the size of its hidden layer is configured to be either $\{10,30,100\}$ or certain values that maintain its total number of parameters as nearly the same as that of SRN. 
We observe that more RNNs failed to learn grammar 6 (with higher complexity) accurately. 
Only 2-RNN, UNI-RNN, and M-RNN have consistently high performance even when their hidden layers have much smaller sizes (e.g., when SRN has a 10-size hidden layer, UNI-RNN, M-RNN, and 2-RNN have their hidden sizes of 5, 6, and 7, respectively.) 
These results validate the above analysis that the second-order hidden interaction is better at learning the transition of DFA. 
In addition, the limitation of other interactions, even more complicated ones used in GRU and LSTM, cannot be easily compensated with more parameters, e.g. both GRU and LSTM with 100 hidden neurons have difficulty in learning grammar 6.

In Figure~\ref{fig:g3_srn_line}-\ref{fig:g3_uni_n_line} are the updates of the hidden weights for SRN, 2-RNN, and UNI-RNN during training. 
The horizontal axis shows the iterations of training, and the vertical axis shows the $L_2$ norm of each weight parameter. 
Specifically, in Figure~\ref{fig:g3_uni_p_line} and \ref{fig:g3_uni_n_line}, we show the results from initializing the hidden layer of UNI-RNN in different ranges ($[-0.52, -0.48]$ and $[0.48,0.52]$). 
Figure~\ref{fig:g3_o2_line}-\ref{fig:g3_uni_n_line} show that the second-order interaction (represented by $W$) dominated the training progress and the update for other weight parameters, i.e., $U$, $V$, and $B$ were negligible. 
However, for the SRN shown in Figure~\ref{fig:g3_srn_line}, both $U$ and $V$ were highly involved during training to compensate for the lack of a second-order interaction.

\subsection{Strictly Local, Strictly Piecewise, and STAMINA Grammars}
\label{sec:eval_sl_sp}

\paragraph{SL \& SP Grammars}
We adopted several sets of strings generated by a SL-$K$ and a SP-$L$ language. 
SL-$K$ is defined with four banned \emph{substrings} with length $K$, and SP-$L$ is defined with one banned \emph{subsequence} with length $L$. 
It is easy to see that neither SL nor SP belongs to the proportional class since they all have one absorbing-rejecting state. 
They can be categorized into either the polynomial or exponential class according to their specific forbidden factors or subsequences. 
We selected the SL-$4$ and SP-$8$ languages (shown in Table~\ref{tab:data}) from the six languages (SL-$K$ and SP-$L$ for $K$ and $L \in \{2,4,8\}$~\footnote{Other languages are omitted since all RNNs easily learn them.}) with the alphabet $\Sigma = \{a,b,c,d\}$ created in the original work~\cite{subregdeep17} and its implementation~\footnote{\url{https://github.com/enesavc/subreg_deeplearning}}. 
These two languages belong to the exponential class and the polynomial class, respectively~\footnote{SL-$4$ is similar to Tomita-3 grammar, and SP-$8$ is similar to Tomita-7 grammar.}. 
For both languages, we selected their 100$k$-size training sets of which each contains random strings between length 1 and 25 (1$k$ and 10$k$ training sets were not used to avoid possible bias caused by insufficient training) to train all RNNs. 
For each language, there are two testing sets (T-1 and T-2) containing strings of the lengths between 1 to 25 and between 26 to 50, respectively. 
All RNNs were trained with up to 300 epochs with the batch size of 100. 

\paragraph{STAMINA Grammars}
We used the string sets provided by the STAMINA competition~\footnote{This provides training sets generated by DFAs with different sizes of alphabets and with varying sparsity (how much they cover the behaviour of the target DFA) for DFA learning.(\url{http://stamina.chefbe.net/home})}. 
We selected 20 (No. 81 - 100) out of the 100 problems from the competition. 
Each of 20 sets was generated for a target DFA with a 50-size alphabet and with the sparsity varying in [12.5\%, 25\%, 50\% 100\%]. 
We show the average BCR scores obtained in sets with the same level of sparsity. 
Since the competition was closed and the ground truth testing labels are not available, for each selected problem, we divided its training set into a new training and a testing set with the ratio of 8:2. 
We trained all RNNs with up to 200 epochs with the batch size of 100.

For both experiments, we only demonstrate the results obtained by controlling all RNNs to have nearly the same number of parameters. 
From the results shown in Table~\ref{tab:results}, we observe that for the SL-$4$, SP-$8$, and grammars from STAMINA, the learning performance of all RNNs follows the same trend. 
Specifically, 2-RNN and UNI-RNN consistently have the best performance, while M-RNN and MI-RNN are generally better than SRN. 
The performance of M-RNN is less stable and may due to the effect of decomposition. 
LSTM and GRU only obtained performance that was comparable with the performance obtained by SRN. 
Prior work~\cite{subregdeep17} has also reported similar results.

\subsection{Word-Level Penn Treebank Data}
We evaluate and compare different RNNs on the language modeling task using Penn-Treebank (PTB) word-level corpus~\cite{MarcusSM94} by their perplexity scores. 
We did not conduct a careful hyper-parameter search on this experiment. 
All RNNs were configured with one hidden layer with different sizes. 
Detailed configurations for all RNNs are provided in the Appendix. 
In the last column of Table~\ref{tab:results}, we show that UNI-RNN achieved better results compared to other RNNs, and it had an evident improvement over the SRN and even comparable with GRU.

%% file: base/6_conclusion.tex
\section{Conclusion}
\label{sec:conclusion}
In order to provide a greater understanding of the relationships between recurrent neural networks (RNNs) and deterministic finite automata (DFA), we performed theoretical analysis and empirical validation for the relations between these two types of models. 
We show that higher-order hidden interaction of a RNN is critical for accurately learning a regular grammar with a high level of complexity. 
This correspondence will hopefully facilitate the analysis of RNNs using DFAs. 
We propose the integration of different RNN orders of interaction into a unified framework, and show that the unified framework is flexible in learning sequential data of various forms. Future work will focus on extensions to other grammars and improvements in training.